\documentclass[11pt,letterpaper]{article}
\usepackage[T1]{fontenc}
\usepackage{times}
\usepackage[margin=1in]{geometry}
\usepackage{amsmath,amssymb,amsthm}
\usepackage{booktabs,graphicx,microtype}
\usepackage{float}
\usepackage{xcolor}
\usepackage{natbib}
\setcitestyle{authoryear,round,citesep={;},aysep={,},yysep={;}}
\usepackage{url,hyperref}
\hypersetup{hidelinks}
\newtheorem{proposition}{Proposition}
\newtheorem{corollary}[proposition]{Corollary}
\newcommand{\KL}{\mathrm{D}_{\mathrm{KL}}}
\newcommand{\TC}{\mathrm{TC}}
\newcommand{\E}{\mathbb{E}}
\newcommand{\Raster}{\textsc{AR-order}}
\newcommand{\LCLM}{\textsc{LCLM-order}}
\newcommand{\Defer}{\textsc{Defer}}
\newcommand{\eol}{\textnormal{\textless eol\textgreater}}
\newcommand{\void}{\textnormal{\textless void\textgreater}}
\floatstyle{ruled}
\newfloat{algorithm}{tbp}{loa}
\floatname{algorithm}{Algorithm}

\title{LatLM: Line-Parallel Language Models with Deferred Commitment}
\author{Shaorong Zhang\textsuperscript{*} \quad Shiyuan Li\textsuperscript{*} \quad Longxuan Yu \quad Greg Ver Steeg\\[0.4em]
\normalsize University of California, Riverside\\[0.3em]
\normalsize \textsuperscript{*}Equal contribution. Correspondence: Greg Ver Steeg (\href{mailto:gregoryv@ucr.edu}{gregoryv@ucr.edu})}

\date{\small Project page: \url{https://szhan311.github.io/latlm-page/}}

\begin{document}
\maketitle

\begin{abstract}
Line-parallel language models advance multiple lines concurrently for substantial speedups. However, frameworks like the Line-Coupled Language Model (LCLM) enforce a fixed commitment schedule across lines. Because dependencies vary within text, a fixed schedule cannot hold back a line whose prerequisites remain uncomputed, forcing premature commitments and cascading reasoning errors. To address this limitation, we introduce the \emph{Lattice Language Model} (\textsc{LatLM}), making commitment schedules an adaptive decoding-time choice. Text occupies a 2D lattice where eligible frontier tokens are scored each round; a policy commits confident candidates while deferring unready ones for extended context. One forward pass supervises entire sampled trajectories via 2D causal masking, while an append-only cache retains exact states.
At 482M parameters, we establish two findings.
(1) LatLM matches a data- and capacity-matched autoregressive model on held-out text when decoded sequentially, while its adaptive parallel decoder achieves near-autoregressive accuracy on independent-line tasks using 25--41\% as many decoding rounds.
(2) When cross-line dependencies cause LCLM's fixed schedule to fail, deferred commitment substantially improves accuracy while retaining parallelism where dependencies permit. On mixed dependency chains, LatLM achieves 96.5\% accuracy versus LCLM's 78.5\%, using fewer than half the decoding rounds of the autoregressive baseline.
\end{abstract}

\section{Introduction}
\label{sec:introduction}

Standard autoregressive decoding generates one token per forward pass, requiring a sequence of model evaluations as output length grows. This strict serialization makes decoding latency scale with the number of generated tokens, motivating methods that produce multiple tokens per step \citep{leviathan2023fast,chen2023accelerating,gloeckle2024better}. The challenge is not merely to predict several tokens at once, but to identify tokens that can be committed concurrently without losing the dependencies that make autoregressive generation reliable.

Existing parallel decoders typically pay for concurrency in state reuse, verification, or dependency handling. Masked diffusion has no causal KV cache, while block diffusion caches only completed blocks \citep{sahoo2024simple,nie2026large,arriola2025block}. Speculative and multi-token methods preserve autoregressive structure, but their realized speedup depends on how many drafted tokens survive verification \citep{gloeckle2024better,cai2024medusa,leviathan2023fast,chen2023accelerating}. A broader dependency cost arises whenever multiple positions or segments advance concurrently: predictions made in the same step cannot condition on one another, and fixed parallel schedules cannot pause an output whose prerequisite has not yet been generated \citep{stern2019insertion,gu2017non,ning2024skeleton,liu2024apar,jin2025learning}.

Parallel commitment is most reliable when concurrently generated tokens are weakly dependent; with exact marginals, its forward-KL cost is characterized by their conditional total correlation \citep{zhang2026generation}. Many parallel decoders seek concurrency along the token axis, but recent work identifies a complementary source across lines, where next tokens can be substantially less dependent than consecutive future tokens \citep{li2026line}. The Line-Coupled Language Model (LCLM) exploits this structure by interleaving lines into a single causal sequence and advancing them in parallel, achieving substantial reductions in decoding steps while retaining a causal model and KV cache.

This success, however, comes from a fixed commitment schedule: LCLM advances every active line by one token per step, whether or not that token's prerequisite information is available. When a token depends on a specific token in an earlier line that has not yet been generated, the model must predict without it. If the missing value cannot be reconstructed from context, it commits a guess prematurely, and subsequent generation may propagate that error. The schedule provides no way to simply wait for the prerequisite to appear.

Our central question is: \emph{Can one model support line-parallel generation
while allowing individual commitments to be deferred?} Our answer is YES. We introduce the
\emph{Lattice Language Model} (\textsc{LatLM}; Figure~\ref{fig:lattice}) to
provide this flexibility. At each round, a policy uses predictive entropy
to select tokens for commitment. Several lines can advance together, or
some can wait for additional context. What \textsc{LatLM} adds is the object the
gate runs on: a frontier that is legal by construction over
variable-length lines, one residual stream per cell whose token enters
only the key/value path, and a cache that stays exact along whatever
trajectory the policy takes, so that holding a cell back costs one
re-score and nothing else. The price is flexibility: a line can only wait
for the lines above it.

\begin{figure}[t]
\centering
\includegraphics[width=\linewidth]{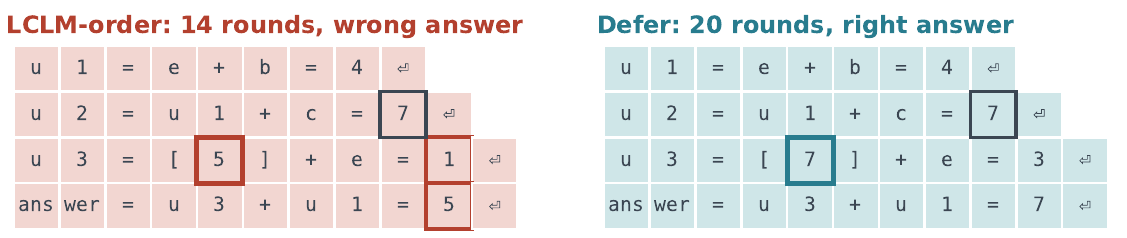}
\caption{\textbf{Why a line sometimes has to wait: one problem, one
checkpoint, two policies.} A mixed chain (Section~\ref{sec:mixed}) decoded
by the same LatLM under the fixed line-parallel schedule (left) and the
entropy gate (right); prompt rows omitted. Row $u_2$ reads $u_1$'s result
in its last digit; row $u_3$ restates $u_2$'s result as a bracketed source digit
(black frame: the result it needs). The fixed schedule reaches that token
in round 7, three rounds before $u_2$'s result exists, copies a wrong
value, and every later token is consistent with it. The gate holds that
one cell (entropy 2.2 nats) while the other rows proceed, and commits it
once $u_2$'s result is visible: 20 rounds and right, against 14 rounds
and wrong; sequential decoding takes 41.}
\label{fig:lattice}
\end{figure}

Our contributions are the model and three findings about it:
\begin{itemize}
    \item \textbf{A line-parallel language model with deferred commitment}. LatLM supports sequential, fixed line-parallel, and adaptive decoding with one set of weights. Its lattice structure allows individual frontier tokens to wait for additional context, while trajectory-based training supervises every cell in one forward pass and an append-only cache remains exact along each legal decoding trajectory.
    \item \textbf{Autoregressive modeling quality with adaptive parallel decoding}. Decoded sequentially, LatLM matches a data- and capacity-matched autoregressive model on held-out text (0.928 vs. 0.931 bits per byte), at approximately 1.3$\times$ the training compute per token. On independent-line tasks, adaptive decoding achieves near-autoregressive accuracy using 25--41\% as many decoding rounds.
    \item \textbf{Improved accuracy over fixed line-parallel decoding under cross-line dependencies}. Deferred commitment recovers near-sequential accuracy where LCLM's fixed schedule fails, retaining parallelism when dependencies permit and falling back toward sequential decoding when necessary. On mixed dependency chains, LatLM achieves 96.5\% accuracy versus LCLM's 78.5\%, using 28.9 rounds versus the autoregressive baseline's 64.8.
\end{itemize}

\section{Lattice Language Models}
\label{sec:method}

\subsection{Generating across lines, with the option to wait}
\label{sec:formulation}

A document is a grid of token cells, one line per row and one token
position per column. The committed cells form a \emph{legal prefix} if
every committed cell's left and upper neighbours are committed, wherever
they exist: each line grows left to right, and a line cannot advance
beyond the line above it (Figure~\ref{fig:lattice_steps}(a)). The
\emph{frontier} is the set of uncommitted cells whose left and upper
neighbours are committed, at most one per line (b). Eligibility is
structural: a frontier cell near the start of one line may still need a
result that has not yet appeared at the end of the line above.

In each round the model predicts a distribution at every frontier cell
from the committed context, and a policy commits a subset (c). The rest
stay open and are predicted again once more context exists. Tokens
committed in the same round cannot condition on one another, so the
choice of subset matters. Committing any subset of the frontier yields
another legal prefix, and generation repeats predict, commit, extend;
legal prefixes ordered by inclusion form a lattice
(Appendix~\ref{app:proofs}).

\begin{figure}[t]
    \centering
    \includegraphics[width=\linewidth]{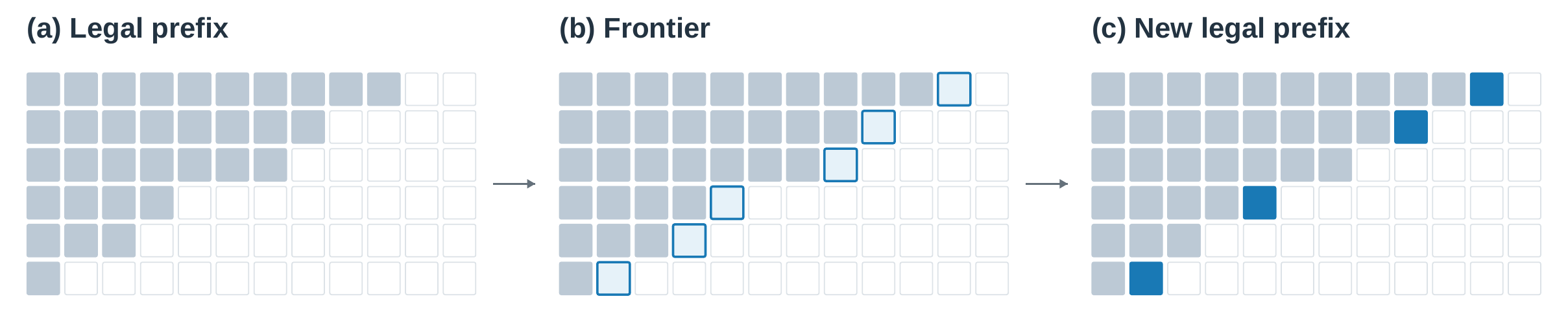}
    \caption{
    \textbf{One round of lattice generation.}
    (a) Gray cells form a legal prefix.
    (b) Blue outlines mark the six frontier positions.
    (c) Accepting four frontier tokens, shown in solid blue,
    produces a new legal prefix.
    }
    \label{fig:lattice_steps}
\end{figure}

\subsection{Mapping text to the grid}
\label{sec:units}

\textit{Text units.}
Each row of the lattice holds a contiguous span of text, which we call a \emph{unit}. Unit boundaries determine which tokens follow one another within a row and which can be generated in parallel across rows. We distinguish \emph{fixed-length units}, cut to a prescribed block size, from \emph{variable-length units} whose boundaries align with sentences or reasoning steps.

\textit{Variable-length units preserve dependencies.}
Fixed boundaries can artificially split dependencies, forcing the model to anticipate unproduced information (e.g., making ``seven'' eligible before ``four'' is committed in Figure~\ref{fig:unit_layout}a). We therefore use variable-length units aligned with natural reasoning steps (Figure~\ref{fig:unit_layout}b). This keeps an answer and its premises within the same row, leveraging the strict left-to-right constraint to ensure operands are committed first without introducing arbitrary cross-unit bottlenecks.

\begin{figure}[t]
    \centering
    \includegraphics[width=\linewidth]{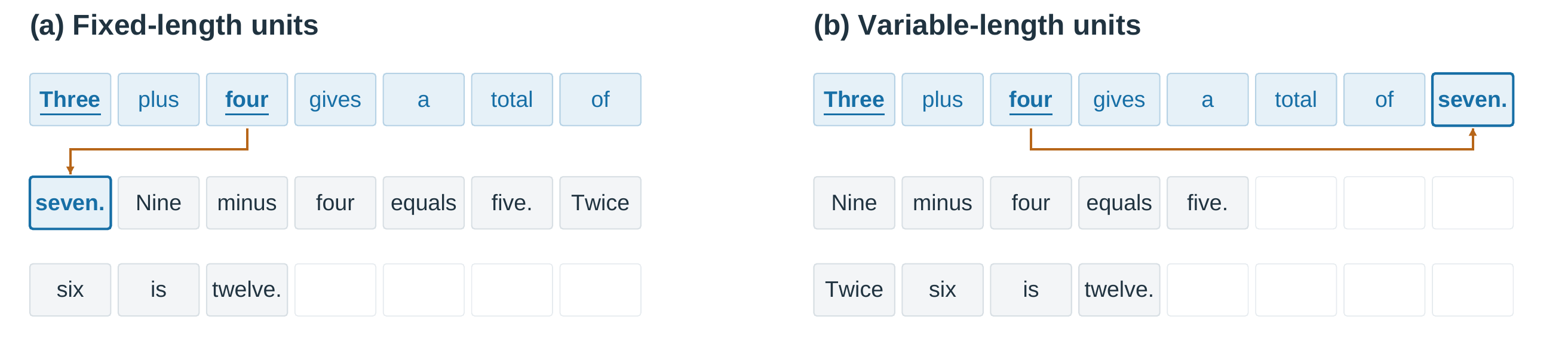}
    \caption{
    \textbf{Unit boundaries change where dependencies fall.}
    The same text is arranged into
    (a) fixed-length units of seven words and
    (b) variable-length units following sentence boundaries.
    Blue cells mark the first sentence; arrows highlight
    the dependency of its answer on an earlier operand.
    Fixed partitioning moves the answer to the next unit,
    whereas sentence-based partitioning keeps the calculation
    together. Cells show words for readability; actual token
    boundaries may differ.
    }
    \label{fig:unit_layout}
\end{figure}

\textit{Two-dimensional positional encoding.} To support
parallel generation of variable-length units, we identify each cell directly by its 2D coordinates $(r,c)$ and apply separate rotary positional encodings \citep{su2024roformer} to the row and column axes. A cell's coordinates thus remain independent of preceding lengths and commitment rounds. Additional coordinate features are described in Appendix~\ref{app:implementation}.

\textit{Unit boundaries and termination.}
Each unit begins with an anchor and ends with a model-predicted \eol{} token. Subsequent positions act as virtual padding omitted from network inputs. Units exceeding the maximum content capacity are simply wrapped into additional rows.

\subsection{Training on legal prefixes}
\label{sec:architecture}
\label{sec:training}

Training predicts frontier tokens from a partially completed document:
select a legal prefix $S$ of a reference, condition on its committed cells
and the prompt, and apply a cross-entropy loss at every frontier cell
(Figure~\ref{fig:training}(a,b)). The targets share one context but carry
their own coordinates, and none may observe another. The committed set \(S\) determines which cells are legal. The history \(h\) also records when committed cells were added, which affects the cached states because cells committed in different rounds see different contexts. Thus, two histories with the same \(S\) can produce different predictions. A prediction stream starts from a learned query embedding at its coordinates, without its token. At each layer, let ($z_v$) denote the residual stream for cell (v), suppressing the layer index, and let ($e(x_v)$) denote its token embedding. The token enters only the key/value path:

\begin{equation}
Q_v=W_Q z_v,\qquad K_v=W_K\bigl(z_v+e(x_v)\bigr),\qquad V_v=W_V\bigl(z_v+e(x_v)\bigr),
\label{eq:routing}
\end{equation}
so a target cannot leak into its own stream and one residual stream per
cell serves both roles (Appendix~\ref{app:implementation}).

\begin{figure}[t]
    \centering
    \includegraphics[width=\linewidth]{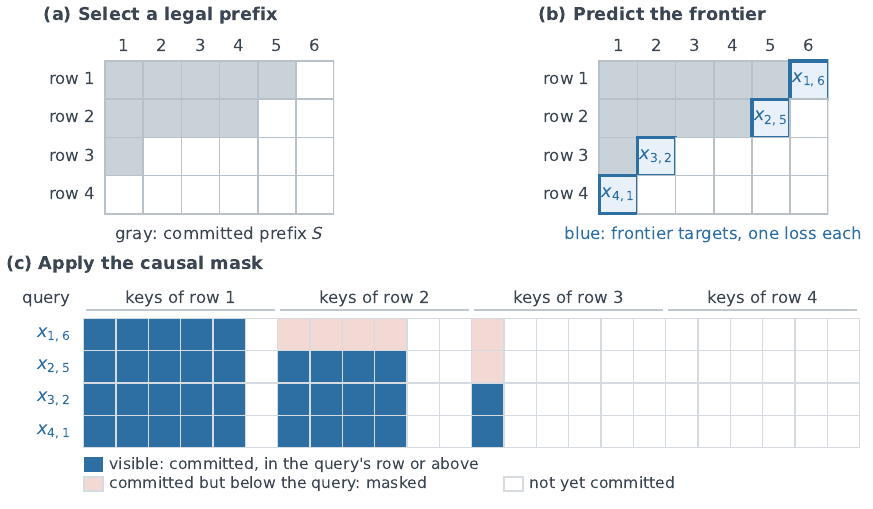}
    \caption{
    \textbf{Training from a legal prefix.}
    (a) Select a legal prefix, shown in gray.
    (b) Predict the frontier tokens and compute their
    cross-entropy losses.
    (c) The frontier-query rows of the attention mask
    (Eq.~\ref{eq:round_mask}): a target attends to committed
    cells in its own row and the rows above it, so $x_{1,6}$
    sees row 1 only and $x_{4,1}$ all ten committed cells;
    committed cells below a query are masked, and so are the
    other targets.
    }
    \label{fig:training}
\end{figure}

\textbf{Packing training contexts into one forward pass.}
To supervise every cell of a document at once we sample a \emph{commit
trajectory} $\tau=(A_1,\ldots,A_T)$, each $A_t$ a nonempty subset of the
current frontier, which assigns each cell $v$ a round $t_\tau(v)$ and a
history $h_\tau(v)$ of the prompt and the cells committed in earlier
rounds. Cells are arranged in round order with their two-dimensional
coordinates, and a query at $u$ attends to $v$ when
\begin{equation}
    M_{uv}
    =
    \mathbf{1}\{t_\tau(v)<t_\tau(u)\}\;
    \mathbf{1}\{0 \le r(u)-r(v) < W\},
    \label{eq:round_mask}
\end{equation}
where $r(\cdot)$ is the row and $W$ a row window ($W=6$ in pretraining, 32
or 64 in the task fine-tunes; the prompt rows are always within reach of the
window as configured). The first factor is causal: a cell sees only what
was committed in earlier rounds, never its own token or a concurrent one.
The second restricts attention to the cell's own row and the $W-1$ rows
above it, so a cell does not read rows below it even once they are
committed. Each cell is supervised once, at the history immediately before its
commitment, and not at the earlier histories in which a policy might have
scored and rejected it; all rounds are computed in one forward pass, the
rounds setting attention visibility, not separate passes.

The training objective is
\begin{equation}
    \mathcal{L}(\theta)
    =
    \mathbb{E}_{x\sim\mathcal{D},\,
        \tau\sim\Pi(\cdot\mid x)}
    \left[
        \sum_{v\in V_{\mathrm{target}}}
        w_v
        \bigl[-\log q_{\theta,v}
        (x_v\mid h_\tau(v))\bigr]
    \right].
    \label{eq:training}
\end{equation}

Targets include content
and unit termination; unit-start targets have weight four, others one. Pretraining samples five schedules with equal
probability: row-major, column-major, simultaneous frontier updates,
random nonempty frontier subsets, and a serialized diagonal order (one
cell at a time by increasing row plus column). Prompt cells take their
rounds from the sampled schedule like any other cell, with their content
loss masked (row-ending supervision is not) so the model cannot fit the
rules' phrasing without reading them, whereas inference prefills the
prompt. This schedule mixture exposes the model to diverse legal histories, though adaptive decoding may encounter histories not sampled during training.

\subsection{Decoding strategies}
\label{sec:policies}

One set of weights supports three decoding strategies
(Figure~\ref{fig:policies}). They use the same eligibility rule but
produce different commitment schedules: each round's choices determine
the context and eligible positions in subsequent rounds.
\Raster{} commits the first row-major frontier cell, finishing one line
before starting the next.
\LCLM{} commits the first $\min(K,|F(S)|)$ frontier cells each round,
advancing multiple lines on a fixed rule regardless of confidence.
\Defer{} scans the frontier in row-major order and commits the first
up to $K$ cells whose predictive entropy
$H_{\theta,v}(h)=-\sum_{a}q_{\theta,v}(a\mid h)\log q_{\theta,v}(a\mid h)$
is at most a threshold $\eta$, so a line whose next token the model cannot
yet predict with confidence waits while later lines take its place. If no
candidate passes, the first row-major cell is committed so that every
round makes progress. All commitments are irreversible
(Algorithm~\ref{alg:decode}).

\LCLM{} is a fixed-width stand-in for LCLM's schedule, which
advances every active line by one token without a width cap and ends with
a sequential answer tail. LCLM's interleaved causal sequence also
lets an earlier line read a later line's committed tokens; our row-causal
mask never allows this (Appendix~\ref{app:schedule}).

\begin{figure}[t]
\centering
\includegraphics[width=\linewidth]{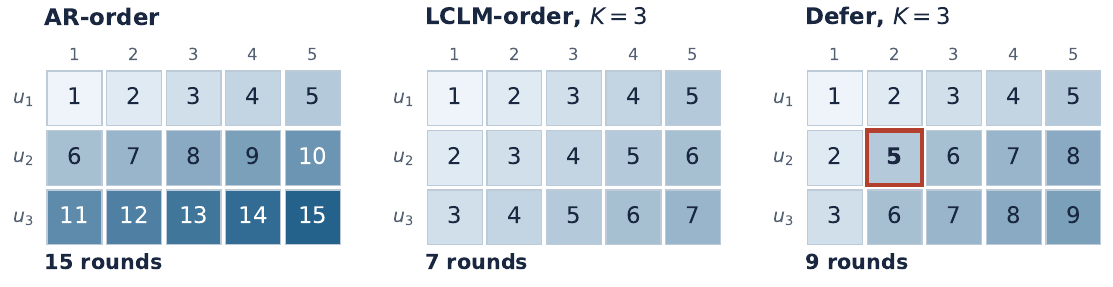}
\caption{\textbf{Three decoding strategies on the same output layout.}
Each cell shows its commitment round; equal numbers denote parallel
generation. \Raster{} completes lines in order. \LCLM{} ($K=3$)
advances all eligible lines. \Defer{} ($K=3$) selectively waits:
in this schematic, cell $(2,2)$ exceeds the entropy threshold in rounds
3--4 and passes on re-scoring in round 5; all other candidates pass.
Line 3 still advances in round 3 while that cell waits. Every schedule
respects left and upper prerequisites from earlier rounds. The layout
omits prompt and structural cells; the gate scores are illustrative.}
\label{fig:policies}
\end{figure}

For these tasks, $\eta$ is the $q$-quantile of entropies observed during
\LCLM{} decoding of 32 or 64 validation problems, without their answer
labels; each task and width has its own calibration, and we report
explicit quantiles rather than a ``best lossless'' point chosen on test
accuracy. A quantile fixes a threshold, not the subsequent acceptance
fraction or error rate. The quantile featured for a task is the loosest
of the calibrated set at which no error is observed on the same
calibration problems, the one use of their labels
(Appendix~\ref{app:gatebaselines}); zero errors on 32 problems is
consistent with a 9\% error rate at the 95\% level, so this is a choice
among a few quantiles, not a guarantee.
Low entropy is a readiness heuristic, not a
dependence detector: exact marginals bound within-round dependence by
their summed entropies, but learned marginals can be confidently wrong
(Appendix~\ref{app:proofs}).

\subsection{Incremental execution and state reuse}
\label{sec:cache}

Committed cells contribute keys and values to an append-only cache. Each
round fuses the previous round's commits and the current frontier queries
into one layerwise pass, querying at most one cell per open row, so the
commit width changes how many queries are kept, not how many are
computed; static buffers and CUDA-graph replay reduce overhead, as they do
for the AR reference. Optionally, a frontier query's per-layer states are
stored and, once it is committed, reused to build its keys and values
without a second attention and feed-forward pass; this is exact in exact
arithmetic because the stream already excludes its own token, though BF16
rounding can change trajectories (Appendix~\ref{app:reuse}). Exactness is
relative to the realized trajectory: a committed cell's state depends on
the history it was committed under, so the same committed set reached in
two orders need not give the same predictions (Appendix~\ref{app:proofs}).

\section{Experiments}
\label{sec:experiments}

One pretrained lattice, fine-tuned per task, is asked to do three things
that usually pull apart: model text as well as an autoregressive model of
the same size (Section~\ref{sec:lm}), advance independent lines in
parallel as LCLM does (Section~\ref{sec:simple}), and stay correct where
LCLM's fixed schedule fails (Section~\ref{sec:failure}).

\textbf{Tasks.}
We use tasks that separate the benefits and risks of line-parallel generation. Held-out text tests whether the lattice architecture preserves ordinary language-modeling quality. Records and list batch contain independent output lines, testing whether adaptive decoding retains parallelism when no cross-line waiting is required. Mixed chains, branching traces, and state tracking introduce explicit cross-line token dependencies, where a token may become eligible before the value it depends on has been generated; these tasks test whether deferred commitment can recover accuracy while preserving whatever parallelism the dependency structure allows.

\textbf{Setup.}
AR, LCLM and LatLM have 482M parameters and were pretrained on the
same ClimbMix split \citep{diao2026nemotron} with one tokenizer, each for
about 13.1B content-token exposures. LCLM is our port of its
interleaved layout and waterfall decoder. Each task gets a short
fine-tune per model at the same budget; LatLM's policies share one
checkpoint per task. Every model predicts its own row endings, gate
thresholds are frozen on validation problems, and we count decoder
rounds; request latency is a separate question (Section~\ref{sec:speed}).
Protocol details are in Appendix~\ref{app:configuration}.

\subsection{Sequential content modeling}
\label{sec:lm}
\label{sec:parity}

Does the lattice pay for its flexibility with worse language modeling? It
does not. Decoded row by row it matches the autoregressive twin on
held-out text, 0.928 against 0.931 bits per byte with a paired interval
that excludes zero (Table~\ref{tab:parity}), for about $1.3\times$ the
training compute per token. The flexibility is paid for only when it is
used: decoding every frontier cell at once costs the same weights about
2\%, and LCLM pays about 3\% under its own layout, in both cases at
the first tokens of each line, where a line starts before the one above
has ended (Appendix~\ref{app:boundary}). Free-generation wall-clock
measurements for the same model pair appear in Section~\ref{sec:speed}.

\begin{table}[t]
\centering
\caption{\textbf{Held-out content modeling at matched capacity and data.}
482M parameters, about 13.1B content tokens each, the same 1,024 windows;
gap intervals are paired over documents; LCLM is scored under its own layout.}
\label{tab:parity}
\begin{tabular}{lcc}
\toprule
Model & BPB $\downarrow$ (95\% CI) & Relative gap (95\% CI) \\
\midrule
Autoregressive reference & 0.9313 [0.9180, 0.9454] & --- \\
LatLM, \Raster{} & 0.9277 [0.9144, 0.9417] & $-0.39\%$ [$-0.59$, $-0.19$] \\
LatLM, parallel frontier & 0.9479 [0.9346, 0.9619] & $+1.77\%$ [$+1.53$, $+2.02$] \\
LCLM, own layout & 0.9580 [0.9446, 0.9720] & $+2.86\%$ [$+2.65$, $+3.06$] \\
\bottomrule
\end{tabular}
\end{table}

\subsection{Simple tasks: keeping line parallelism}
\label{sec:simple}
\label{sec:reasoning}

Records asks the model to verbalize several prompt-provided records, one per output line, while list batch asks it to solve one independent addition problem per line. Because no output line depends on another, both tasks test whether adaptive decoding preserves line parallelism when waiting is unnecessary. Here the option to wait should cost nothing, and on
records it does not: the gate finishes in LCLM's round count, a quarter
of AR's, at full accuracy (Table~\ref{tab:simple}). List batch shows the
trade the gate makes when lines are independent but individually harder.
LCLM keeps its rounds and loses five points; the gate spends half
again as many rounds and stays within a point of sequential decoding.
Where it is not confident it spends rounds rather than accuracy, which is
the behaviour a fixed schedule cannot have.

\begin{table}[t]
\centering
\caption{\textbf{Simple independent tasks.} Accuracy (\%) and mean rounds
on 400 fresh problems per task, own row endings; LatLM width six, one
frozen threshold per task (records $q=1.0$, list $q=0.95$).
Ratios report round-count reduction, not wall-clock speedup.}
\label{tab:simple}
\small
\setlength{\tabcolsep}{4pt}
\begin{tabular}{lcccccc}
\toprule
& \multicolumn{3}{c}{Records} & \multicolumn{3}{c}{List batch} \\
\cmidrule(lr){2-4}\cmidrule(lr){5-7}
Method & Acc. & Rounds & vs.\ AR & Acc. & Rounds & vs.\ AR \\
\midrule
AR & 100.0 & 130 & 1.00$\times$ & 100.0 & 129 & 1.00$\times$ \\
LCLM & 99.8 & 33 & 3.89$\times$ & 94.8 & 34 & 3.74$\times$ \\
LatLM, \Raster{} & 100.0 & 130 & 1.00$\times$ & 98.8 & 129 & 1.00$\times$ \\
LatLM, \LCLM{} & 100.0 & 33 & 3.94$\times$ & 77.5 & 33 & 3.96$\times$ \\
LatLM, \Defer{} & 100.0 & 33 & 3.93$\times$ & 98.2 & 53 & 2.44$\times$ \\
\bottomrule
\end{tabular}
\end{table}

\subsection{Dependent computations under fixed and adaptive schedules}
\label{sec:failure}

We test if the option to wait recovers accuracy where LCLM's fixed schedule fails. We evaluate on mixed chains (where dependent cells are scattered, allowing some parallelism) and branching traces (where every row depends on the last, severely limiting parallelism). All tasks are scored on full-trace accuracy.

\textbf{Dependence scattered through the output.}
\label{sec:mixed}
Mixed chains interleave independent, dependent, and restating rows (which copy a previous result mid-trace). The models diverge strongly here (Table~\ref{tab:mixed}). LCLM reaches restated inputs before the previous row writes them, losing 21\% of traces. The autoregressive model (AR) achieves 100\% accuracy but pays maximum rounds (64.8). The entropy gate, however, achieves 96.5\% accuracy in 28.9 rounds and 98.5\% in 33.2, about half the AR round count. Contrasting LatLM's own fixed schedule (29.0\%) with its gated schedule explicitly isolates this scheduling benefit. The gate succeeds by selectively holding most prematurely eligible restated digits until their inputs arrive, while allowing independent rows to advance. This accuracy stems from smart selection, not simply taking more rounds: randomly deferring cells or matching the gate's per-round counts with a fixed row-major order both fail (Appendix~\ref{app:gatebaselines}).

\textbf{Budget and density.}
Increased training budgets (2$\times$ or 4$\times$) help LCLM recover shallow dependencies, but deep dependencies still fail; the gate makes few errors at any tested depth by waiting instead of recomputing (Appendix~\ref{app:budget}). Furthermore, as dependence density increases from 30\% to 100\% of rows, a gated fine-tune automatically adjusts its rounds from one-third to one-half of AR's to maintain near-sequential accuracy. At the highest densities, one of two LatLM seeds has no error-free calibration quantile. Conversely, LCLM's accuracy collapses. Even a single checkpoint with a frozen threshold adapts gracefully to varying densities (Appendix~\ref{app:density}).

\textbf{Chains where every row waits.}
When every row sequentially depends on the previous one, the critical path spans the entire chain. In branching traces, successors must agree with the chosen operation of the row above. LCLM's accuracy drops rapidly as branches increase, while LatLM with \Defer{} remains accurate (Figure~\ref{fig:chains}). Similarly, in 32-step state tracking, LCLM fails at deeper steps because its waterfall must recompute running states from the prompt (Appendix~\ref{app:state}). LatLM with \Defer{} stays close to sequential accuracy on both tasks, though it sacrifices nearly all parallelism to do so---decoding about one cell per round in strict dependency order. While an oracle scheduler could still achieve 60--80\% of sequential rounds by parallelizing templated row prefixes, the gate successfully secures the necessary serial fallback that a fixed schedule cannot offer.

\textbf{What a dependent cell does when its input is missing.}
A probe (Appendix~\ref{app:reveal}) reveals that if a previous row's result is missing, a dependent cell often copies the most recent visible result; state step cells are an exception because they can recompute one lookup. If a wrong value is planted, it copies that instead. The lattice reads but does not verify. The gate prevents these confident-but-wrong copies by enforcing waits, though its entropy heuristic is imperfect: loosened thresholds let exactly these false copies slip through (Appendix~\ref{app:schedule}).

\begin{table}[t]
\centering
\caption{\textbf{Mixed chains: every policy on one LatLM checkpoint, and
the two baselines, on the same 200 fresh problems.} Full-trace accuracy
(\%) and mean rounds. ``Answer tail'' holds only the last row;
``row-major at the gate's counts'' replays the featured gate's per-round
counts with row-major cells. Bold: the loosest quantile with no error on
the 32 calibration problems.}
\label{tab:mixed}
\small
\setlength{\tabcolsep}{5pt}
\begin{tabular}{lcc}
\toprule
Policy & Acc. & Rounds \\
\midrule
AR & 100.0 & 64.8 \\
LCLM & 78.5 & 25.6 \\
LCLM, 24,000 updates & 92.0 & 25.6 \\
\midrule
LatLM, \Raster{} & 99.5 & 64.8 \\
LatLM, \LCLM{} & 29.0 & 22.2 \\
LatLM, \LCLM{} + answer tail & 29.0 & 29.4 \\
LatLM, random veto 0.5 & 27.0 & 38.1 \\
LatLM, row-major at the gate's counts & 33.5 & 31.5 \\
LatLM, \Defer{}, $q=0.9$ & 98.5 & 33.2 \\
LatLM, \Defer{}, $q=0.95$ & \textbf{96.5} & \textbf{28.9} \\
LatLM, \Defer{}, $q=0.99$ & 29.0 & 22.2 \\
\bottomrule
\end{tabular}
\end{table}

\begin{figure}[t]
\centering
\includegraphics[width=0.86\linewidth]{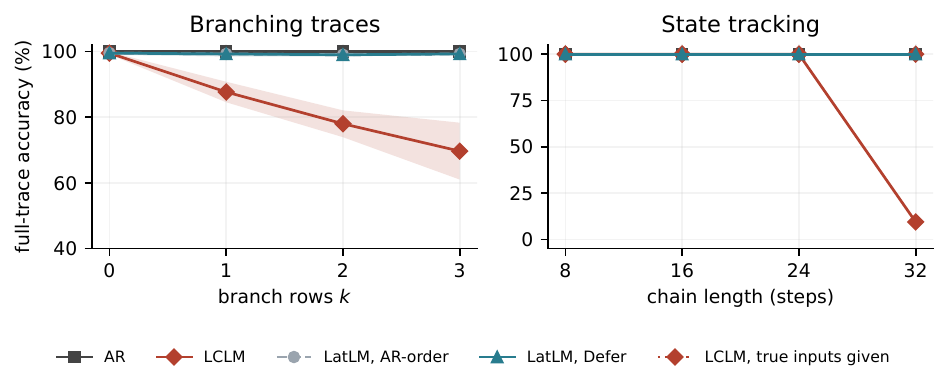}
\caption{\textbf{Chains where every row waits: the gate is right where LCLM
is not.} Left: branching traces, full-trace accuracy against the number
of branch rows $k$ (AR and LCLM: mean over three seeds, band one
standard deviation; LatLM two seeds). Right: state tracking, accuracy
against chain length; the dotted line gives LCLM the true inputs in the
prompt. LatLM with \Defer{} matches its own sequential decode everywhere, at
sequential cost.}
\label{fig:chains}
\end{figure}


\subsection{Wall-clock decoding cost}
\label{sec:speed}

Table~\ref{tab:wallclock} reports wall-clock costs for the 482M pair.
Single-request text generation with static caches and CUDA graphs takes
1.81 ms/cell under sampled width-six \LCLM{} and 2.54 under \Defer{},
versus AR's 3.59 ($1.98\times$ and $1.41\times$). Timing includes prompt
seeding, scheduling and sampling, excluding graph setup. These are
three repeats of one prompt with differing outputs, not a matched-quality
latency comparison.

\begin{table}[t]
\centering
\caption{\textbf{Wall-clock costs (482M, BF16, H100 NVL).}
Left: batch-one sampled text; speedup compares ms/cell.
Right: synthetic kernels using FlashAttention, ms/AR step or ms/LatLM round.
Protocols and full results: Appendix~\ref{app:reuse}.}
\label{tab:wallclock}
\small
\setlength{\tabcolsep}{3pt}
\begin{minipage}[t]{0.48\linewidth}
\centering
\begin{tabular}[t]{lrr}
\toprule
Decoder & ms/cell & Speedup \\
\midrule
AR & 3.59 & 1.00$\times$ \\
LatLM, \LCLM{} & 1.81 & 1.98$\times$ \\
LatLM, \Defer{} & 2.54 & 1.41$\times$ \\
\bottomrule
\end{tabular}
\end{minipage}\hfill
\begin{minipage}[t]{0.50\linewidth}
\centering
\begin{tabular}[t]{rrrr}
\toprule
Batch & AR & LatLM & + reuse \\
\midrule
1 & 3.08 & 3.57 & 3.80 \\
64 & 4.48 & 6.56 & 5.77 \\
256 & 8.34 & 20.11 & 13.19 \\
512 & 13.58 & 38.94 & 24.19 \\
1024 & 24.27 & 80.09 & 48.55 \\
\bottomrule
\end{tabular}
\end{minipage}
\end{table}

At batch 1024, query-state reuse reduces LatLM's round cost from
80.09 to 48.55 ms ($1.65\times$), versus AR's 24.27 ms/step: the
kernel-level break-even is about two commits per round. This synthetic
benchmark fixes six query/commit slots and a 200-position visible prefix,
excluding dynamic scheduling and sampling. Task accuracy and held-out
BPB are evaluated separately from these timed operating points.

\section{Related Work}
\label{sec:related}

\paragraph{Flexible autoregressive order and structured generation.}
Any-order autoregressive models train under multiple orderings
\citep{shih2022training}; A3 \citep{du2026autoregressive} is the closest
inference procedure, committing a confidence- or entropy-selected subset
of positions each step and re-predicting the rest, over arbitrary groups
with two-stream attention \citep{pynadath2026any}; Esoteric LMs cache
exactly under causal masks \citep{sahoo2025esoteric}. The gate is
therefore not new; what the lattice adds is the partial order, one
residual stream per cell with the token entering only the key/value
path, a cache exact along the realized trajectory, and the within-model
controls of Section~\ref{sec:failure}. Skeleton-of-Thought, APAR, PASTA,
ASPD, Multiverse, Hogwild! and Multi-Stream LLMs generate response parts
as parallel branches or streams
\citep[surveyed by][]{zhang2025survey,ning2024skeleton,liu2024apar,jin2025learning,chen2025aspd,yang2026multiverse,rodionov2026hogwild,su2026multi}.
The Line-Coupled
Language Model \citep{li2026line} runs a waterfall schedule as a plain
causal model, without the option to hold a line back.

\subsubsection*{Parallel token selection and verification}
Masked and block diffusion predict multiple positions
\citep{ghazvininejad2019mask,sahoo2024simple,nie2026large,arriola2025block};
Fast-dLLM's confidence threshold \citep{wu2026fast} plays the role of
our gate, over a bidirectional fixed-size block with an approximate
cache; set diffusion updates an exact cache after every flexible-position
set rather than every block \citep{arriola2026set}; lattice rows are
variable-length, ordered by dependence rather than a diffusion schedule,
and committed keys are never recomputed.

Later commit rules use the confidence profile
\citep{kasa2026fast}, an auxiliary AR model \citep{israel2026accelerating},
estimated dependencies
\citep{ringel2026dependency,zoabi2026mean,park2026dependency}, a
self-containedness test \citep{wang2026commit}, or a learned policy
\citep{jazbec2025learning}; \citet{cai2026confidence} analyse when
confidence-based decoding is efficient, and \citet{asaria2026neither}
measure how a diffusion LM commits, as Appendix~\ref{app:schedule} does
for the gate; ReFusion fills parallel
slots autoregressively \citep{li2026refusion} and ParallelBench measures
dependence-induced failures \citep{kang2026parallelbench}. Training under many legal histories, as
our schedule mixture does, is shared with MBD-LM, D2F, NAP,
Jacobi forcing and certainty forcing
\citep{jin2026multi,wang2026diffusion,li2026diffusion,hu2025fast,chen2026dparallel}.
Unlike speculative sampling \citep{leviathan2023fast,chen2023accelerating}
or parallel token prediction \citep{draxler2026parallel}, lattice commitment
carries no verification guarantee.

\section{Discussion and Conclusion}
\label{sec:limitations}

\textsc{LatLM} allocates sequential computation according to dependence.
At 482M it matches the autoregressive twin's held-out BPB, retains
LCLM-level round reductions on independent lines, and recovers
near-sequential accuracy where fixed schedules fail. The cost approaches
sequential decoding on serial chains and half of AR's rounds on mixed
chains. Wall-clock measurements show per-cell gains
in text generation and reduced batched kernel cost with query-state reuse.
Task-specific fine-tunes, limited seeds, and cross-model differences in
representation and training preclude a universal threshold or LCLM failure
boundary. Matched-quality serving speedups and a readiness signal for
confident-wrong copies remain open.

\bibliographystyle{plainnat}
\bibliography{references}

\clearpage
\appendix

\section{Structural and Probabilistic Details}
\label{app:proofs}

\subsection{Legal sets and the frontier}
Consider a rectangular grid $X=\{X_{r,c}\}$ with $R$ rows and $C$ columns,
indexed from one, and define $(r,c)\preceq(r',c')$ when $r\le r'$ and
$c\le c'$. A committed set $S$ is legal if it is downward closed under
$\preceq$: whenever a cell is committed, all its predecessors are
committed. Equivalently, every committed cell has its left and upper
neighbors committed, when they exist. The frontier is
\begin{equation}
 F(S)=\{v\notin S:\text{all strict predecessors of }v\text{ belong to }S\}.
 \label{eq:frontier}
\end{equation}

\begin{proposition}[Legal parallel updates]
\label{prop:legality}
The legal sets form a distributive lattice under union and intersection.
For a legal $S$, $F(S)$ is an antichain with at most one cell per row and
per column. Committing any $A\subseteq F(S)$ preserves legality.
\end{proposition}
\begin{proof}
Let $S_1,S_2$ be order ideals. If $v\in S_1\cup S_2$, then $v$ belongs
to at least one ideal, which contains all predecessors of $v$; hence
$S_1\cup S_2$ is an ideal. If $v\in S_1\cap S_2$, its predecessors
belong to both ideals, so their intersection is also an ideal. Union
and intersection are the join and meet under inclusion and satisfy the
distributive laws of sets.

Suppose distinct $u,v\in F(S)$ satisfy $u\preceq v$. The frontier
definition requires every strict predecessor of $v$, including $u$, to
belong to $S$, contradicting $u\in F(S)$. Thus $F(S)$ is an antichain.
Cells in one row, or in one column, are comparable, giving at most one
frontier cell per row and column. Finally, all predecessors of any
$v\in A\subseteq F(S)$ already lie in $S$. Adding all of $A$ therefore
preserves downward closure, irrespective of their token values.
\end{proof}

\begin{corollary}[Geometric parallel depth]
For an initially empty $R\times C$ grid with no extra constraints and
unlimited commit width, at least $R+C-1$ rounds are required, and this
number is attained by committing every frontier cell at each round.
\end{corollary}
\begin{proof}
A monotone path from $(1,1)$ to $(R,C)$ contains $R+C-1$ cells, each a
predecessor of the next; no two can be committed in one round. This
gives the lower bound. The schedule assigning round $r+c-1$ to cell
$(r,c)$ is legal and attains it.
\end{proof}
This bound concerns the rectangular precedence graph only. A prompt,
virtual padding, width limits, and additional dependency constraints
change the decoding problem. It is not a lower bound on the number of
neural evaluations required to solve a reasoning task by other means.

\subsection{Histories and normalized generation}
Let $h_t$ contain the prompt and all earlier committed cells, their values,
and their grouping into rounds; write $S_t$ for its committed set. A policy
selects a nonempty $A_t\subseteq F(S_t)$ with $|A_t|\le K$, and a
stochastic round uses the factorization
\begin{equation}
 Q_{\theta,A_t}(x_{A_t}\mid h_t)
 =\prod_{v\in A_t}q_{\theta,v}(x_v\mid h_t),
 \qquad S_{t+1}=S_t\cup A_t.
 \label{eq:commit}
\end{equation}
Conditioning on $h_t$ rather than $S_t$ matters: cached representations
retain their processing history, so different trajectories reaching the
same committed set need not produce the same predictions. For a fixed
trajectory, the product in Equation~\eqref{eq:commit} over
all rounds defines a normalized distribution over a finite grid because
each round uses normalized conditionals on previous rounds. A deterministic
adaptive policy with nonempty updates also defines a normalized sequential
process when it depends only on the observed history and model predictions,
not on unseen target values. Summing over next-round outcomes recursively
verifies normalization. The per-token conditionals learned under different
histories need not be marginals of a common, order-independent joint
distribution. An adaptive schedule therefore need not preserve the
distribution induced by \Raster{} even when both use the same weights.

\subsection{Dependence and marginal prediction error}
\label{sec:dependence}
For a fixed history $h$ and candidate set $A$, a standard decomposition
separates the cost of independent commitment from marginal prediction error:
\begin{equation}
 \KL\!\left(P_A\,\middle\|\,\prod_{v\in A}q_{\theta,v}\right)
 =\underbrace{\TC_P(X_A\mid h)}_{\text{dependence within the committed set}}
 +\underbrace{\sum_{v\in A}\KL(P_v\|q_{\theta,v})}_{\text{marginal prediction error}},
 \label{eq:decomposition}
\end{equation}
where $P_A$ is the target joint distribution, $P_v$ its marginals, and
$\TC_P=\KL(P_A\|\prod_vP_v)$ the total correlation; all distributions
condition on $h$.
\begin{proposition}
\label{prop:decomposition}
For categorical distributions with finite divergence and a fixed history
$h$, Equation~\eqref{eq:decomposition} holds.
\end{proposition}
\begin{proof}
Suppress conditioning on $h$. Add and subtract the log of the product
of target marginals:
\begin{align}
 \KL\!\left(P_A\,\middle\|\,\prod_{v\in A}q_{\theta,v}\right)
 &=\E_{P_A}\left[\log\frac{P_A(X_A)}{\prod_{v\in A}P_v(X_v)}\right]
   +\sum_{v\in A}\E_{P_A}\left[\log\frac{P_v(X_v)}{q_{\theta,v}(X_v)}\right]\\
 &=\TC_P(X_A\mid h)+\sum_{v\in A}\KL(P_v\|q_{\theta,v}).
\end{align}
The last equality follows by marginalizing the expectation in each term.
\end{proof}
For finite categorical variables,
$\TC_P(X_A\mid h)=\sum_{v\in A}H_P(X_v\mid h)-H_P(X_A\mid h)$, and
nonnegativity of joint entropy gives $\TC_P\le\sum_vH_P(X_v\mid h)$, the
bound behind the entropy heuristic of Section~\ref{sec:policies}. Even
exact marginals lose within-round dependence when sampled independently.
Conversely, a unique deterministic answer given the prompt has zero target
total correlation, yet a finite model may need intermediate results to
compute it; the tasks of Section~\ref{sec:failure} test this reliance on
context, not an information-theoretic need to reason serially.

\section{Architecture and Execution Details}
\label{app:implementation}

\begin{algorithm}[t]
\caption{Entropy-gated commitment for a \textsc{LatLM}}
\label{alg:decode}
\small
\textbf{Input:} Given history $h$, legal set $S$, commit width $K$, threshold $\eta$.
\begin{enumerate}
\item Commit every legal structural cell without a forward pass: the anchor
      at column one of a newly legal row, and \void{} padding after a row's
      \eol{} (or after its last column), which closes the row.
\item Set $F\leftarrow F(S)$, ordered by row and then column. Stop if $F$ is empty.
\item Compute $q_{\theta,v}(\cdot\mid h)$ and $H_{\theta,v}(h)$ for all $v\in F$.
\item Let $A$ be the first $K$ cells of $F$ satisfying $H_{\theta,v}(h)\leq\eta$.
      If $A=\varnothing$, select the first row-major cell instead.
\item Predict each $x_v$, $v\in A$, using its distribution from step 3.
\item Append $(A,x_A)$ to $h$, set $S\leftarrow S\cup A$, and return to step 1.
\end{enumerate}
\textbf{Output:} Generated content.
\end{algorithm}

\paragraph{Token routing.}
With $z_v^{(\ell)}$ the residual stream at layer $\ell$ and $e(x_v)$ the
cell's token embedding, the projections are those of Eq.~\ref{eq:routing}
at every layer, up to normalization and positional rotation. The query's
initial residual contains no $x_v$. By induction over
layers, its attention can receive token identities only from strictly
earlier rounds, which cannot have attended to $v$ in their own updates.
Tokenwise feed-forward layers preserve this property. For frontier
inference, no token embedding is needed for an uncommitted query, since
its key and value are not visible to another current query. The
implementation uses special handling for given seed positions, whose
values are already part of the prompt.

\paragraph{Coordinates and structure.}
The checkpoint combines two-axis RoPE with normalized absolute
coordinate features and boundary flags; it is not purely translation
invariant in the row coordinate. In pretraining and held-out text scoring,
column normalization uses a fixed
configuration value of 66 rather than the current document's maximum
realized length, preventing that feature from revealing future row
lengths. Task fine-tuning and prompt prefill normalize columns by the
packed grid width; the reported task decodes use 66 for incremental
scoring and cache updates. Rows longer than the content limit are wrapped. Once \eol{}
is produced, subsequent row padding is structural and omitted from
transformer inputs. Task evaluations supply row anchors and the output row
count; the model predicts \eol{}, and generated \eol{} cells are included
in reported rounds.

\paragraph{Two gating implementations.}
Algorithm~\ref{alg:decode} specifies the reference reasoning policy:
if every candidate exceeds the threshold, it commits the first row-major
cell. The fused free-generation implementation used in
Table~\ref{tab:speed} instead falls back to up to $K$ lowest-entropy
candidates. It also predicts row termination and uses the fixed
2.5-nat threshold in that benchmark. Its timing must therefore not be
assigned to the reasoning policy without a matching evaluation.

\section{Training and Evaluation Configuration}
\label{app:configuration}

\paragraph{Corpus and token budget.}
The corpus is ClimbMix \citep{diao2026nemotron}, with an 8k BPE tokenizer.
The corpus has approximately 14.2M training documents and
13.13B content tokens, with a held-out document split. The lattice
samples six-row windows; the causal reference samples chunks of up to
512 tokens. Both sample documents uniformly, followed by offsets within
documents. The training budget counts sampled content-token exposures,
approximately one corpus token count; random window sampling does not
mean that every corpus token is visited exactly once.

\paragraph{Architecture and optimizer.}
Both 482M models use 24 transformer layers, hidden dimension 1280,
16 attention heads, feed-forward dimension 5120, AdamW, weight decay
0.1, gradient clipping at 1.0, batch size 96, and bf16 training.
The causal reference trains for 350,500 steps, uses a peak learning
rate of $10^{-3}$ after 200 warm-up steps, and decays linearly over its
last 15\% of steps. The lattice is a continuation: initial pretraining
uses $6\times10^{-4}$; adaptation to the padding-free layout and
supervised row termination uses $5\times10^{-5}$; the final continuation
keeps that rate before a final linear decay. The combined lattice exposure is approximately 13.1B
content tokens. Row starts receive weight four in its training loss;
evaluation is unweighted. Neither configuration was selected through
a learning-rate sweep at the final budget.

\paragraph{BPB convention.}
For document $i$, let $L_i$ be summed unweighted content-token NLL and
$B_i$ its byte denominator. We report
$(\sum_iL_i)/(\log 2\sum_iB_i)$. The evaluator computes $B_i$ as the sum
of UTF-8 lengths obtained by decoding each scored token individually;
both models share this denominator. This tokenwise convention may
differ from byte counting after decoding the entire concatenated text.
Bootstrap resampling operates on paired document-level $(L_i,B_i)$
records. Scores exclude structural-token losses and therefore are
content modeling scores, rather than full-document code lengths
including structural generation.

\paragraph{Training cost.}
Matched content-token exposure is not matched compute. Timed on one H100
at batch 96 with each model's pretraining configuration, the lattice
trainer processes about $1.3\times$ fewer content tokens per GPU-second
than the autoregressive twin, roughly 130 against 100 single-GPU hours
for the two runs. The measured ratio is not explained by parameter count alone: a committed
cell's token-to-key/value routing (Eq.~\ref{eq:routing}) runs the shared
query/key/value projection twice, on the residual stream with and without
the token embedding, which adds roughly a quarter to each block's linear-layer
cost relative to one autoregressive step; the packed two-dimensional attention
and the six-row windows change the attention cost separately.

\paragraph{Fine-tuning cohorts.}
Every task fine-tune starts from the final lattice checkpoint (or the
matched AR and LCLM checkpoints) and trains all three models for the
same number of updates on the same generator: 3,000 for records and list,
6,000 for the branching traces, 12,000 for state tracking and the modular
and mixed chains. Test problems use a separate generator split, with
offset 10,000 for the simple tasks, modular and mixed chains, and fresh-set
confirmations; the state and branching evaluations use offset zero.
There are 400 problems per simple task, 200 per condition
otherwise, with 32 test problems per depth in the state scan. The state
and branching scans use a 64-column cap and nine prompt rows; the modular
and mixed chains 32 columns and four prompt rows. LatLM upweights
row-start losses by four; AR and LCLM use unweighted losses. In a mixed
chain each row after the first is independent with probability $0.4$,
dependent with $0.3$ or restating with $0.3$, and the answer adds the last
result to an earlier independent one.

In the LCLM task decoder, the answer-row anchor enters the cache
in the round that starts the last reasoning row, one round earlier than
in its training layout. Cross-model task comparisons therefore also
include this difference in training and decoding context.

\paragraph{\Defer{} calibration.}
Thresholds are entropy quantiles observed under the \LCLM{} policy on
validation problems (64 for records and list, 32 elsewhere), then frozen;
within each scan the checkpoint, width, layout and threshold are fixed
across test conditions. Every model generates its own row endings with
the output row count specified. Equal aggregate accuracies do not imply
identical per-example outputs.

\section{Additional Task Results}
\label{app:additional}

\subsection{Decoding examples}
\label{app:examples}
Figures~\ref{fig:ex_mixed}--\ref{fig:ex_text} show single decodes with
every cell shaded by the round it was committed in. On the mixed chain
(Figure~\ref{fig:ex_mixed}) the restating row needs the value the row
above has not yet written. LCLM's waterfall reaches that token
first and restates a wrong digit; the rest of its trace is consistent
with the wrong value. LatLM's fixed schedule does the same. The gate
holds that one cell until the value above exists and commits the
rest of the row on schedule, so it finishes in half the sequential
rounds with the right answer. On the records task
(Figure~\ref{fig:ex_records}) nothing needs to wait: the gate advances
every line together and produces the same text as sequential decoding in
a third of the rounds. On held-out text (Figure~\ref{fig:ex_text}) the
gate commits about two cells per round; the cells it holds longest are
line starts and the openings of clauses, which is where
Appendix~\ref{app:boundary} finds the waterfall's loss.

\begin{figure}[ht]
\centering
\includegraphics[width=\linewidth]{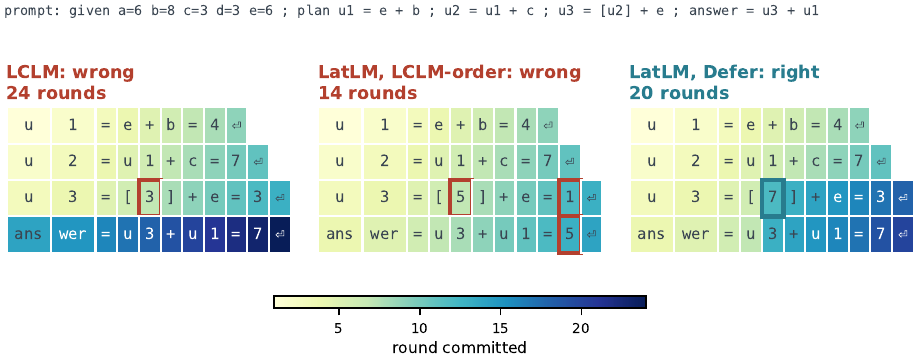}
\caption{\textbf{One mixed chain under three decoders.} Cells shaded by
commit round (LCLM: the emission order its layout implies). Red
frames mark tokens that differ from the reference; the teal frame marks
the restated digit the gate held. LatLM with \Raster{} takes 41 rounds.}
\label{fig:ex_mixed}
\end{figure}

\begin{figure}[ht]
\centering
\includegraphics[width=\linewidth]{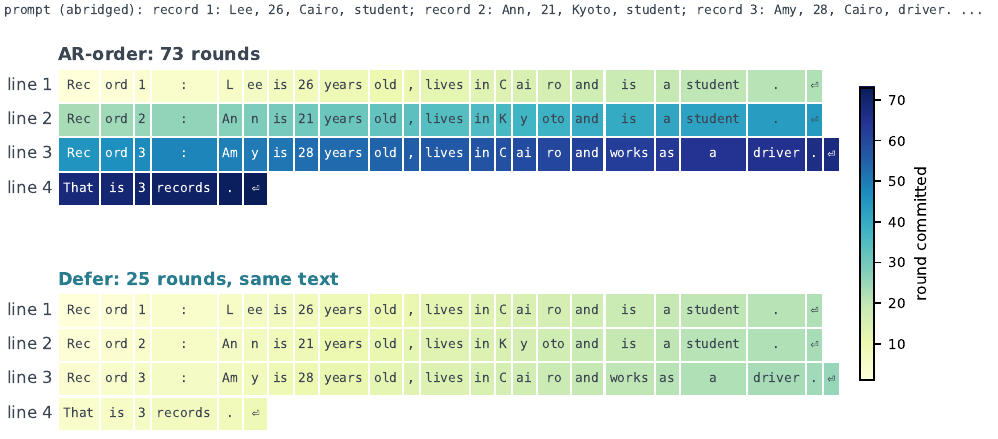}
\caption{\textbf{One records problem, \Raster{} and \Defer{}.} Independent
lines advance together under the gate; the text is identical.}
\label{fig:ex_records}
\end{figure}

\begin{figure}[ht]
\centering
\includegraphics[width=0.96\linewidth]{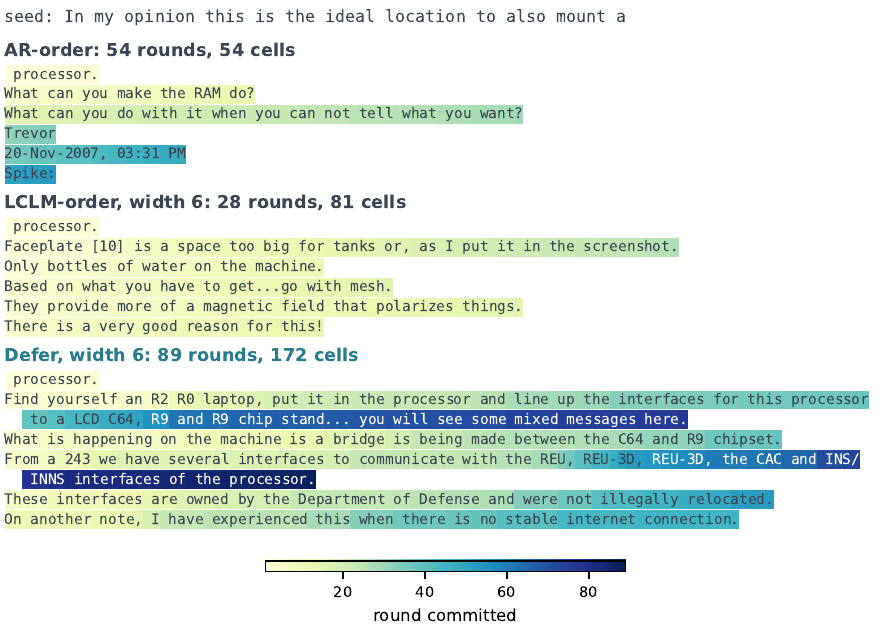}
\caption{\textbf{One held-out seed continued by three policies of the
same checkpoint} (sampled, temperature one, top-$p=0.95$; each policy
draws its own continuation). Shading gives the commit round; the gate
uses the generation study's threshold.}
\label{fig:ex_text}
\end{figure}

\subsection{State tracking across depths}
\label{app:state}
A prompt gives two permutations over eight states, a start state and a
sequence of operations; each output row records its input and the result
of one lookup, so every row needs the result of the last. Through 24
steps every model is correct and LCLM needs a fifth of the rounds
(Figure~\ref{fig:chains}, right). At 32 steps LCLM fails on nearly every trace;
given the true inputs in the prompt it completes them all, so what it
loses is not the lookup but the running state, which its waterfall must
recompute from the prompt before the previous row has finished. LatLM with \Defer{} is correct throughout with both training seeds, the separation
replicates on four times as many fresh problems, and its cost is the full
sequential round count. LCLM's failure point moves with training budget,
so it is a limit at this budget rather than a bound on depth, and row
count and depth are confounded.
\subsection{Sampling diagnostic for the fork scan}
\label{app:sampling}
Under argmax a near-tie at a branch is broken by a deterministic function
of the logits, and the successor row, reading the same prompt with the
same weights, breaks it the same way; consistency is then free without
reading anything. Sampling (temperature one, top-$p=0.95$, thresholds
recalibrated under sampling) makes the branch a real coin flip. LCLM's
accuracy at three branches roughly halves relative to greedy decoding and
nearly all of its wrong rows are branch rows; the autoregressive model is
unaffected, and LatLM with \Defer{} stays within a few points of its own
sequential decode. This is the sense in which the failure is
informational rather than computational: the value is not in the prompt.

\subsection{Reveal probe: what a dependent cell does when its input is missing}
\label{app:reveal}
On the fork and state LatLM checkpoints, we take a dependent row's source
cell (or the answer row's number) and score the frontier once under
teacher-forced histories that differ only in the previous row's result:
\emph{present}; \emph{wrong} (another admissible value planted in that cell);
\emph{absent} (the previous row committed through its operator, result cell
empty); and \emph{holding} (rows above complete, the previous row committed
only through the column that makes the probed cell legal -- the state the gate
is in while it holds a row). 32 fresh test problems, four dependent rows each
plus the answer row; entropy is computed as in the evaluation decode.
Table~\ref{tab:reveal} reports the probability of the correct value, of the
planted value, and of the most recent result visible on the grid (the row
above the predecessor's). A diagnostic that uses the reference to build
histories, not a decoding method.

\begin{table}[ht]
\centering
\caption{Reveal probe. Mean probability of the correct source value, the
planted wrong value, and the most recent visible result; median entropy.
The first and third can coincide when the most recent visible result is
the correct value, so the columns are not disjoint.
Answer rows are the final ``Answer'' / ``ends with'' cells.}
\label{tab:reveal}
\small
\setlength{\tabcolsep}{4pt}
\begin{tabular}{llcccc}
\toprule
Cells & History & $p$(correct) & $p$(planted) & $p$(most recent visible) & Entropy \\
\midrule
fork, step & present & 1.000 & -- & -- & 0.001 \\
fork, step & wrong & 0.001 & 0.994 & -- & 0.001 \\
fork, step & absent & 0.122 & -- & 0.730 & 0.483 \\
fork, answer & holding & 0.004 & -- & 0.916 & 0.107 \\
state, step & present & 1.000 & -- & -- & 0.000 \\
state, step & wrong & 0.104 & 0.896 & -- & 0.024 \\
state, step & absent & 0.985 & -- & 0.128 & 0.015 \\
state, answer & holding & 0.209 & -- & 0.900 & 0.224 \\
\bottomrule
\end{tabular}
\end{table}

Three things follow. The lattice reads the previous row's result rather
than deriving it: with the result present it is exact, and a planted
wrong value is copied. With the result missing it writes the most recent
result it can see, the row above the predecessor's, on most fork step
cells and on essentially every answer row; state step cells are the
exception because a one-step table lookup is recomputable from the
visible operator. And the gate's threshold has to be as tight as it is
because these copies are often confident: in a looser greedy decode of
the state task every loss was an answer row equal to the second-to-last
result, and the probe at the holding history had that copied value as its
argmax in every one of them.

\subsection{When only one cell must wait: modular chains}
\label{app:modular}
Modular chains make the dependence local. A plan of three to eight steps
adds a named constant to the previous value modulo ten; every token of a
row but its result digit is a copy from the plan, and a final row copies
the last value. The prompt states the plan, so LCLM, recomputing
the running value from the prompt, is lossless here; the task isolates
what the gate does when exactly one cell per row must wait. LatLM's fixed
schedule fails on the answer row and almost nowhere else, because that
row is written before the last result exists. Two controls isolate this effect. Holding
the answer row until every other row has ended, LCLM's own answer tail,
recovers the accuracy at the cost of the whole row; the gate recovers it
within half a round of the fixed schedule, because it holds only the
digit that is waiting (Figure~\ref{fig:curves}, third panel).

\subsection{\Defer{} versus random deferral at matched cost}
\label{app:gatebaselines}
Is the gate's recovery a matter of deferring more cells or of deferring
the right ones? Figure~\ref{fig:curves} places every policy of one
checkpoint per task on the accuracy--rounds plane, with two controls.
Random deferral vetoes every frontier candidate independently with a
fixed probability, under the same progress fallback; it approaches
sequential accuracy only at nearly sequential cost. Row-major replay commits, in
each round, as many cells as the gate committed in that round of the same
problem, but the first cells in row-major order; it matches the gate's
rounds approximately and collapses to the fixed schedule's
accuracy. At the parallel operating points, matching the commit count alone
does not recover the gate's accuracy; the choice of cells matters.

Two further controls (Figure~\ref{fig:curves}) place the gate between a
fixed rule and an oracle that sees the dependency edges but not the
values. A \emph{line stagger}, delaying each row a fixed number of
columns behind the one above, matches the gate only on the modular
chains, where a single cell waits; on the other three tasks it is safe
only once it is nearly sequential. The oracle, which holds every
dependent cell until the row it reads has ended, is lossless everywhere
and marks how much of the near-sequential cost is the task's rather than
the gate's: on the mixed and modular chains the gate is already close to
it, but on the branching traces it captures only a small share of the
oracle's saving over sequential decoding, and on the state chains none
at all, where entropy's only safe quantile is fully sequential.

\begin{figure}[ht]
\centering
\includegraphics[width=0.92\linewidth]{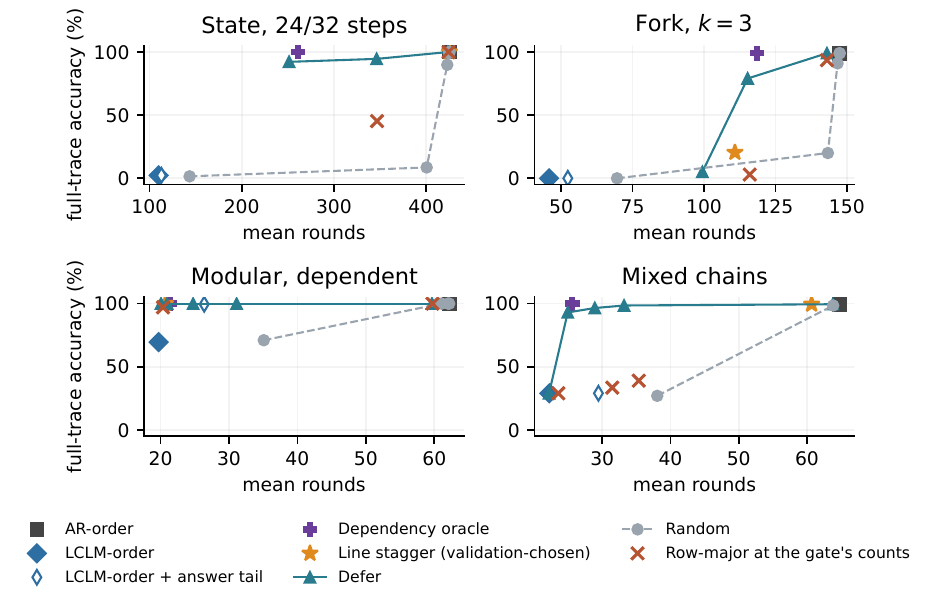}
\caption{\textbf{Accuracy against rounds for every policy of one
checkpoint.} One panel per dependent task, same fresh problems within a
panel; left and up is better.
\Defer{} points are the calibrated quantiles, random points the veto rates;
single runs, no error bars. On the state and fork chains only the
near-sequential end is accurate; on the modular chains the gate is
lossless down to the fixed schedule's cost; on the mixed chains it trades
a few points for less than half the rounds where the fixed schedule
collapses. Crosses: the gate's per-round commit counts replayed with
row-major cells; plus signs: the dependency oracle; stars: the line
stagger chosen on validation.}
\label{fig:curves}
\end{figure}

\subsection{What the gate does at each operating point}
\label{app:schedule}
Analysis of commitment schedules shows two regimes. At the
median-entropy threshold of the state and branching tasks it is
sequential decoding in all but name: about one cell per round, the
fallback firing in nearly half the rounds, and no dependent cell ever
committed before its input. Its accuracy there comes from staying in
dependency order; random deferral can also approach this accuracy at
nearly sequential cost. At looser thresholds the gate selects. On the branching task nearly a
thousand input cells reached the frontier before their result; the gate
held all but a few dozen until the result arrived, and every trace in
which it let one through failed. On the mixed chains at the featured
quantile it held nineteen in twenty of the restated inputs that reached
the frontier early, which is where its accuracy at half the rounds comes
from. Where a cell it lets through is recomputable, as a state one step
behind is, the trace usually survives; where it is not, it does not.

\subsection{Mixed chains across dependence densities}
\label{app:density}
The mixed chains of Section~\ref{sec:mixed} draw each row after the
first as independent with probability $1-\rho$ and dependent or
restating otherwise ($\rho=0.6$ in the main text). Figure~\ref{fig:density}
repeats the protocol across densities, one fine-tune per model and
density: LCLM's accuracy falls steadily, while the gate stays within a
few points of sequential accuracy, its rounds rising from about a third
of the autoregressive model's to about half. Calibration on 32 problems
is not infallible: at the two highest densities, one of two LatLM seeds
has no quantile free of calibration error and contributes no \Defer{} point
rather than a misleadingly safe one, and at $\rho=0$, where no row but
the answer reads a written value, that fine-tune never learns the merged
answer even though a $\rho=0.6$ checkpoint solves the same problems --
our reading is that the lookup gets too little training signal when
nothing else in the output needs a written value, a limitation of the
training scheme rather than the decoders. The more direct test of
adaptivity is the bottom row: one checkpoint and one threshold, both
fixed at $\rho=0.6$, applied without retraining to every other density.
Its rounds still track the density it meets while its accuracy stays
close to sequential, whereas the same LCLM checkpoint collapses
from lossless to under a third -- the gate adapts to the dependence it
meets, the fixed schedule cannot.

\begin{figure}[ht]
\centering
\includegraphics[width=0.9\linewidth]{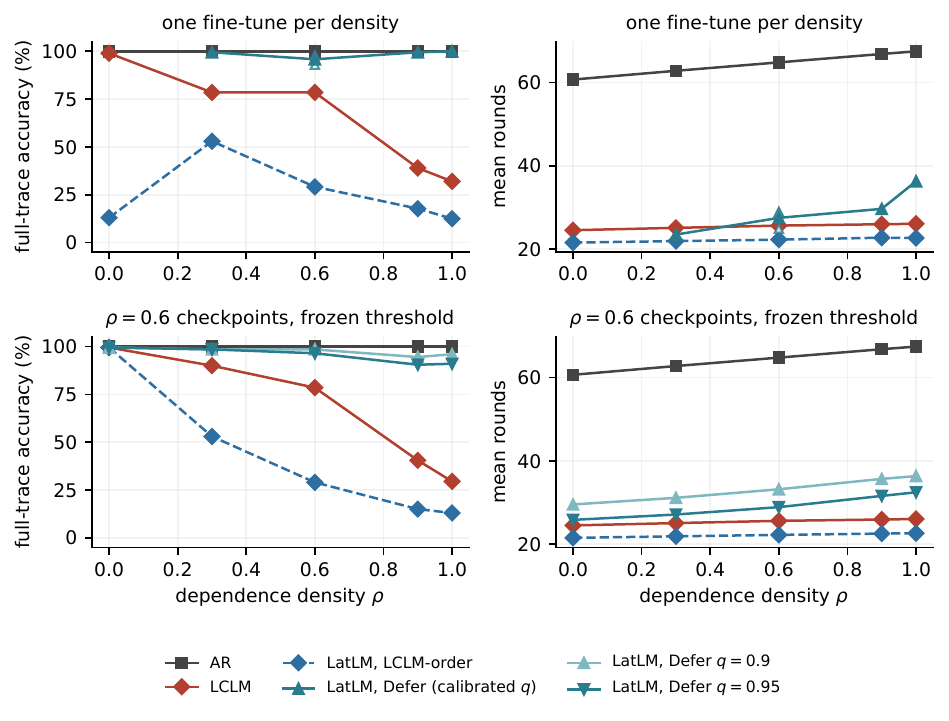}
\caption{\textbf{Mixed chains by dependence density.} Full-trace accuracy
(left) and mean rounds (right) against the fraction $\rho$ of rows that
need the previous result. Top: one fine-tune per model and density, the gate at
its calibrated quantile. One fine-tune per model and density; where
several LatLM seeds exist the \Defer{} line is their mean and hollow markers
the seeds. Bottom: the $\rho=0.6$ checkpoints of Table~\ref{tab:mixed}, with
LatLM's threshold frozen at its $\rho=0.6$ calibration, applied to the
other densities.}
\label{fig:density}
\end{figure}

\subsection{Training budget and recomputation depth on the mixed chains}
\label{app:budget}
LCLM's waterfall reaches a restating row's bracketed source digit before the
row above has written its result, so the value it restates must be
recomputed from the nearest values the waterfall can see. Fine-tuned for
twice and four times the updates, LCLM rises from 78.5\% to 92\% and
stays there. Figure~\ref{fig:recompute} locates the errors by the depth of
the recomputation, the number of dependent rows since the last
independent one: more training removes the shallow errors and leaves the
deep ones, while the gate, which waits for the value instead of
recomputing it, has almost none at any depth. Doubling LatLM's own budget
leaves its sequential decode and its gate at the featured quantiles within
a point of the 12,000-update checkpoint; only the loosest operating
points move.

\begin{figure}[ht]
\centering
\includegraphics[width=0.92\linewidth]{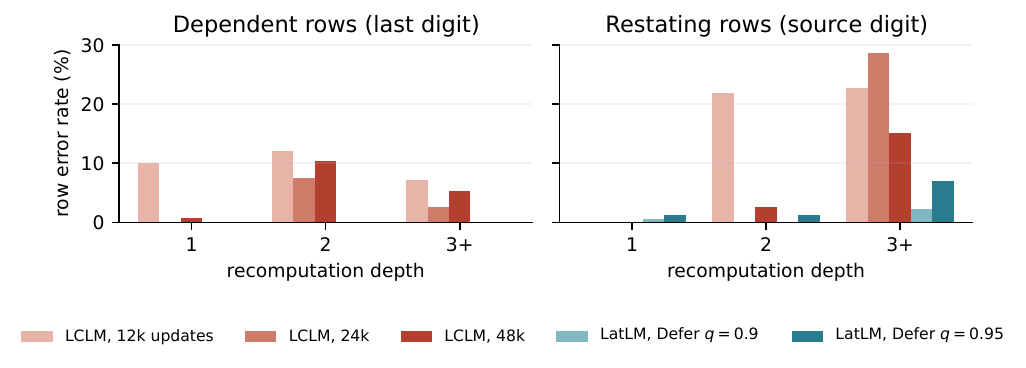}
\caption{\textbf{Row error rate by recomputation depth on the mixed
chains.} Depth $d$: the row is the $d$-th dependent or restating row since
the last independent row; rows are counted while the trace above them is
still correct. Dependent rows need the previous result in their last
digit, restating rows in their bracketed source digit. AR and LatLM with \Raster{} make
no errors and are omitted.}
\label{fig:recompute}
\end{figure}

\subsection{Held-out loss by position within the line}
\label{app:boundary}
Scoring the held-out windows of Table~\ref{tab:parity} by the position
of each token within its line locates the waterfall's cost where LCLM's
own analysis does \citep{li2026line}: at the first token of a line,
which is generated before the line above has ended, and to a lesser
degree at the next three; LatLM's own schedules match autoregressive decoding from the fifth token
on; LCLM, scored under its own layout, does not -- it remains about
0.02 bits per byte above AR there, a gap that is small per token but, summed
over the many tokens at that position, is a majority of its total excess
loss over AR. A teacher-forced entropy gate recovers this loss
from the inside out: a loose threshold restores the later positions and
half of the second to fourth while leaving the line start untouched, and
only a threshold that makes the decode sequential reaches the line start.
A decoding-time rule does better there: letting a line start only once
the line above is a few cells ahead of the waterfall recovers most of the
line-start loss at two cells per round, and at the same total loss uses
fewer rounds than the gate alone. Entropy is one policy the lattice can
run, and on free text not the best one.

\subsection{Free text generation}
The generation study uses 48 held-out prompts, a 12-token seed, a
$6\times64$ grid and top-$p=0.95$ sampling, scored by an earlier
autoregressive checkpoint as a likelihood judge. Sampled width-six
continuations raise the judge's loss by about 6\% relative to sequential
lattice samples. Greedy width-six generation scores better on the judge
and much worse on distinct trigrams: a favourable likelihood can
accompany repetitive output, so the greedy throughput figure is not
quality-preserving acceleration. The \Defer{} row uses its own threshold,
distinct from the latency benchmark and the reasoning calibrations.

\section{Wall-Clock Protocols and Query-State Reuse}
\label{app:reuse}

\paragraph{Single-request text generation.}
The matched 482M models use BF16 weights, static caches and CUDA graphs
on an H100 NVL. We use the first 12 tokenizer tokens of ``The city
council met on Tuesday to discuss the new budget.'' as the prompt and a
$6\times64$ lattice. Sampled rows use temperature 1, top-$p=0.95$ and
random seed 0; \Defer{} uses an entropy threshold of 2.5 nats, committing
up to six lowest-entropy candidates if none passes the threshold.
After a setup call allocates buffers and captures the graph, we take
the median of three complete generation calls, synchronizing CUDA
before and after each timed call. Timing includes per-request prompt
seeding, scheduler logic and sampling, but excludes setup. Table~\ref{tab:speed}
divides this measured duration by the number of generated cells,
including row endings and excluding the prompt. Output lengths differ:
AR and the fixed schedules end all six rows; \Defer{} ends five, with
the remaining row reaching the grid boundary. These timings do not measure
LCLM or share the \Defer{} threshold of the 48-prompt quality study;
the task decoder also uses a different progress fallback.

\begin{table}[ht]
\centering
\caption{\textbf{Single-request wall-clock latency.} H100 NVL, 482M, BF16,
$6\times64$ grid, 12-token prompt, medians of three calls.
Speedup is the ratio of AR to LatLM ms/cell; outputs differ across policies.}
\label{tab:speed}
\small
\setlength{\tabcolsep}{4pt}
\begin{tabular}{lrrrrr}
\toprule
Decoder & Rounds & Cells & ms/round & ms/cell & Speedup \\
\midrule
AR, static cache + graph & 191 & 191 & 3.59 & 3.59 & 1.00$\times$ \\
LatLM, \Raster{} greedy & 66 & 66 & 5.32 & 5.32 & 0.67$\times$ \\
LatLM, \LCLM{} width 6, greedy & 20 & 84 & 8.62 & 2.05 & 1.75$\times$ \\
LatLM, \LCLM{} width 6, sampled & 59 & 181 & 5.56 & 1.81 & 1.98$\times$ \\
LatLM, \Defer{} width 6, sampled & 102 & 198 & 4.94 & 2.54 & 1.41$\times$ \\
\bottomrule
\end{tabular}
\end{table}

\paragraph{Batched kernel timing.}
We time the same model pair with a fixed 200-position visible prefix,
six query and six commit slots per LatLM round, and batches from 1 to
1024 (Table~\ref{tab:batched-wallclock}). Both models use BF16,
FlashAttention over the contiguous visible prefix, and CUDA graphs.
Following three warm-up calls and graph capture, we report the median
of 20 graph replays, with CUDA synchronized before and after each
wall-clock measurement. Inputs and caches are synthetic; the prefix
mask does not exercise arbitrary row-window masks. These measurements
exclude prompt processing, sampling and dynamic batch scheduling.

\begin{table}[ht]
\centering
\caption{\textbf{Batched wall-clock kernel costs.} Milliseconds per AR step
or LatLM round on H100 NVL. Reuse speedup compares the fused LatLM
kernel with query-state reuse, not LatLM generation with AR generation.
Break-even is reuse ms/round divided by AR ms/step: the required mean
commits per round per sequence to match AR's kernel throughput.}
\label{tab:batched-wallclock}
\small
\setlength{\tabcolsep}{6pt}
\begin{tabular}{rrrrrr}
\toprule
Batch & AR & LatLM fused & LatLM reuse & Reuse speedup & Break-even \\
\midrule
1 & 3.08 & 3.57 & 3.80 & 0.94$\times$ & 1.23 \\
64 & 4.48 & 6.56 & 5.77 & 1.14$\times$ & 1.29 \\
256 & 8.34 & 20.11 & 13.19 & 1.52$\times$ & 1.58 \\
512 & 13.58 & 38.94 & 24.19 & 1.61$\times$ & 1.78 \\
1024 & 24.27 & 80.09 & 48.55 & 1.65$\times$ & 2.00 \\
\bottomrule
\end{tabular}
\end{table}

\paragraph{Query-state reuse.}
State reuse retains the normalized input to each attention layer for
frontier queries; after selection, the relevant states are combined with
the committed token embedding and projected into keys and values, so a
committed cell's attention and feed-forward layers are not recomputed.
Deferred candidates use their latest scoring states. Reuse adds overhead
at batch one but reduces round cost at larger batches. At batch 1024,
the kernel-throughput ratio to AR would be
$c\,t_{\mathrm{AR}}/t_{\mathrm{reuse}}\simeq c/2$ for a mean of $c$
committed cells per round; $c$ is not measured by this synthetic benchmark.
On shared trajectories reuse changes no argmax in FP32 and three of
515 in BF16, and its held-out loss shift is within the bootstrap interval
of zero.

\end{document}